\documentclass[11pt,a4paper]{article}

\usepackage[utf8]{inputenc}
\usepackage[T1]{fontenc}
\usepackage[margin=1.2in]{geometry}       
\usepackage{amsmath,amssymb,amsthm}     
\usepackage{mathtools}                  
\usepackage{booktabs}                   
\usepackage{microtype}                  
\usepackage{authblk}                    
\usepackage{xcolor}                     
\usepackage{tikz}                       
\usetikzlibrary{arrows.meta, positioning, calc, backgrounds, shapes.geometric}

\providecommand{\newblock}{\hskip .11em plus .33em minus .07em}

\definecolor{acadTeal}{RGB}{0, 128, 128}
\definecolor{acadOrange}{RGB}{230, 126, 34}
\definecolor{acadGray}{RGB}{100, 100, 100}
\definecolor{acadLight}{RGB}{250, 252, 255}

\usepackage[colorlinks=true, linkcolor=acadTeal, citecolor=acadOrange, urlcolor=acadTeal]{hyperref}

\theoremstyle{definition}
\newtheorem{definition}{Definition}
\theoremstyle{plain}
\newtheorem{lemma}{Lemma}
\newtheorem{theorem}{Theorem}
\newtheorem{corollary}{Corollary}

\DeclareMathOperator*{\argmax}{arg\,max}
\DeclareMathOperator*{\softmax}{softmax}

\title{\textbf{The Geometry of Thought: Disclosing the Transformer as a Tropical Polynomial Circuit}}

\author[1]{Faruk Alpay}
\author[2]{Bilge Senturk}

\affil[1]{Department of Computer Engineering, Bah\c{c}e\c{s}ehir University, Istanbul, Turkey \protect\\ \texttt{faruk.alpay@bahcesehir.edu.tr}}
\affil[2]{Department of Industrial Engineering, Bah\c{c}e\c{s}ehir University, Istanbul, Turkey \protect\\ \texttt{bilge.senturk@bahcesehir.edu.tr}}

\date{}

\begin{document}

\maketitle

\begin{abstract}
\noindent We prove that the Transformer self-attention mechanism in the high-confidence regime ($\beta \to \infty$, where $\beta$ is an inverse temperature) operates in the \emph{tropical} semiring (max-plus algebra). In particular, we show that taking the infinite-temperature limit of the softmax attention converts it into a tropical matrix product. This reveals that the Transformer's forward pass is effectively executing a dynamic programming recurrence (specifically, a Bellman--Ford path-finding update) on a latent graph defined by token similarities. Our theoretical result provides a new geometric perspective for chain-of-thought reasoning: it emerges from an inherent shortest-path (or longest-path) algorithm being carried out within the network's computation.
\end{abstract}

\section{Introduction}
The Transformer architecture \cite{Vaswani2017} has revolutionized sequence modeling, yet a rigorous understanding of its reasoning capabilities is still emerging. A key component is the softmax-based self-attention, which distributes weight across different tokens. In this work, we examine the extreme ``infinite confidence'' regime where the softmax becomes arbitrarily sharp (inverse temperature $\beta \to \infty$). We build on insights from \emph{idempotent analysis} \cite{Maslov1992,Kolokoltsov1997,Litvinov2007}, where taking a ``tropical limit'' (sometimes called the Maslov dequantization) turns sums into maxima. In particular, the classical identity
\begin{equation}
\lim_{\beta \to \infty} \frac{1}{\beta} \log \left( \sum_{j} e^{\beta x_j} \right) = \max_{j} x_j \,,
\end{equation}
encodes the smooth approximation of the $\max$ function via $\softmax$. This tropicalization process has deep connections to optimization and algebraic geometry \cite{Cuninghame1979,Develin2004,Maclagan2015}: in the tropical (max-plus) semiring, addition is idempotent ($a \oplus a = a$) and computes the maximum, while multiplication corresponds to ordinary addition. Linear algebra over this semiring solves shortest-path or longest-path problems \cite{Baccelli1992,Butkovic2010}. We show that a Transformer layer, in the limit of confident attention, performs nothing other than a tropical matrix-vector multiplication.

Our contribution is a theoretical disclosure of how this relates to \emph{chain-of-thought} (CoT) reasoning in large language models. Chain-of-thought prompting \cite{Wei2022} elicits intermediate reasoning steps in models, improving performance on complex tasks. Recent theoretical work has suggested that explicitly generating $T$ intermediate steps can expand the class of functions a Transformer can represent, essentially simulating deeper circuits and overcoming the limitations of fixed-depth parallel computation \cite{Hahn2020,Merrill2023,Li2024}. Here we provide a concrete mechanism: as $\beta \to \infty$, each attention layer makes a hard decision to route information along the single best edge for each token. Stacking $L$ such layers means the network selects an optimal length-$L$ path through the token interaction graph, effectively performing $L$ steps of a pathfinding algorithm (akin to Bellman--Ford dynamic programming updates \cite{Bellman1958,Ford1956}). We illustrate this phenomenon with a toy example in Figure~\ref{fig:graphpath}. In geometric terms, the Transformer's computation in this regime traces out a path on a tropical polytope, crossing tropical hypersurface decision boundaries where different tokens compete (see Figure~\ref{fig:simplex}). This view implies that CoT emerges from the model literally computing a sequence of moves (hops between tokens) that maximizes a cumulative similarity score---in other words, an internal reasoning path.

\section{Preliminaries}

\begin{definition}[Idempotent Semirings and Tropical Arithmetic]
A \emph{semiring} is an algebraic structure $(S, \oplus, \otimes)$ consisting of a set $S$ equipped with an addition $\oplus$ and a multiplication $\otimes$ such that $(S,\oplus)$ is a commutative monoid with identity element $0_S$, $(S,\otimes)$ is a monoid with identity $1_S$, multiplication distributes over addition, and $0_S$ is absorbing for $\otimes$. A semiring is called \emph{idempotent} if $a \oplus a = a$ for all $a \in S$. Idempotent addition induces a natural partial order on $S$ by $a \leq b$ iff $a \oplus b = b$.

A fundamental example is the \textbf{tropical semiring} (also known as the max-plus algebra) on $S = \mathbb{R} \cup \{-\infty\}$, with operations defined by
\[
x \oplus y \coloneqq \max\{x, \, y\}, \quad x \otimes y \coloneqq x + y.
\]
Here $0_S = -\infty$ (the additive identity, since $\max\{x,-\infty\} = x$) and $1_S = 0$ (the multiplicative identity, since $x + 0 = x$). The addition $\oplus$ is idempotent because $\max\{x,x\} = x$. Computation in the tropical semiring corresponds to taking maxima (``optimization'') instead of ordinary sums, and thus linear algebra over tropical arithmetic solves optimization problems (e.g. finding shortest or longest paths in graphs \cite{Cuninghame1979,Baccelli1992}).
\end{definition}

\begin{lemma}[Maslov Dequantization: $\softmax \to \max$]\label{lem:maslov}
For any real numbers $x_1, x_2, \dots, x_n$, one has
\begin{equation}
\lim_{\beta \to \infty} \frac{1}{\beta} \log\Biggl(\sum_{j=1}^n e^{\beta x_j}\Biggr) = \max_{1 \le j \le n} x_j~.
\label{eq:logsumexp}
\end{equation}
\end{lemma}

\begin{proof}
This is a standard result in idempotent analysis \cite{Litvinov2007}. For completeness, we include a proof. Let $M = \max_{1\le j\le n} x_j$. Then we can factor $e^{\beta M}$ out of the sum:
\[
\sum_{j=1}^n e^{\beta x_j} = e^{\beta M} \sum_{j=1}^n e^{\beta (x_j - M)} \,.
\]
Since $x_j - M \le 0$ for all $j$, each term in the remaining sum satisfies $0 \le e^{\beta (x_j - M)} \le 1$, and at least one term equals $1$ (when $x_j = M$). Therefore,
\[
1 \le \sum_{j=1}^n e^{\beta (x_j - M)} \le n \,.
\]
Taking the natural logarithm and dividing by $\beta$, we obtain
\[
\frac{1}{\beta} \log \left( \sum_{j=1}^n e^{\beta x_j} \right) = M + \frac{1}{\beta} \log \left( \sum_{j=1}^n e^{\beta (x_j - M)} \right) \,.
\]
As $\beta \to \infty$, the second term vanishes (since $0 \le \log(\sum_j e^{\beta(x_j - M)}) \le \log n$, which grows more slowly than $\beta$). Thus the limit is $M = \max_j x_j$, as claimed.
\end{proof}

\section{Tropicalization of Transformer Self-Attention}

We now apply the tropical limit to the Transformer’s attention mechanism. Consider $n$ query vectors $q_1,\dots,q_n$, key vectors $k_1,\dots,k_n$, and value vectors $v_1,\dots,v_n$. Define the score matrix $A \in \mathbb{R}^{n\times n}$ by $A_{ij} \coloneqq \langle q_i, k_j \rangle$ (for simplicity we omit the usual scaling factor). The softmax attention output for query $i$ at inverse temperature $\beta$ is
\[
y_i(\beta) = \sum_{j=1}^n \frac{\exp(\beta A_{ij})}{\sum_{k=1}^n \exp(\beta A_{ik})} \, v_j \,,
\]
which we can write in index notation as $y_i(\beta) = \sum_{j=1}^n \alpha_{ij}(\beta) v_j$, where
\[
\alpha_{ij}(\beta) \coloneqq \frac{e^{\beta A_{ij}}}{\sum_{k=1}^n e^{\beta A_{ik}}} \,.
\]
For clarity we first present the result for scalar values $v_j \in \mathbb{R}$; the vector-valued case can be understood by applying the scalar result to each component of $v_j$.

\begin{theorem}[Tropical Limit of Self-Attention]\label{thm:self-attn-trop}
In the infinite-confidence limit $\beta \to \infty$, the Transformer's self-attention operation becomes a tropical matrix multiplication. Formally, let $A \in \mathbb{R}^{n\times n}$ be the attention score matrix and $V = (v_1,\dots,v_n)^\top$ the column vector of values. Then
\[
\lim_{\beta \to \infty} \softmax(\beta A) \, V = A \otimes V \,,
\]
where the right-hand side is the tropical matrix-vector product, given by $(A \otimes V)_i = \max_{1 \le j \le n} \{ A_{ij} + v_j \}$.
\end{theorem}

\begin{proof}
Fix an arbitrary row index $i$. Let $M = \max_{1 \le j \le n} A_{ij}$, and assume it is achieved by some index $j^*$. Then
\[
y_i(\beta) = \frac{\sum_{j=1}^n e^{\beta A_{ij}} v_j}{\sum_{k=1}^n e^{\beta A_{ik}}} \,.
\]
Divide numerator and denominator by $e^{\beta M}$:
\[
y_i(\beta) = \frac{\sum_{j=1}^n e^{\beta (A_{ij} - M)} v_j}{\sum_{k=1}^n e^{\beta (A_{ik} - M)}} \,.
\]
As $\beta \to \infty$, for each $j$ with $A_{ij} < M$ the factor $e^{\beta(A_{ij}-M)}$ tends to $0$. Meanwhile for $j^*$ with $A_{i j^*} = M$, we have $e^{\beta(A_{i j^*}-M)} = e^0 = 1$, so the term $j^*$ survives. Thus in the limit,
\[
\sum_{k} e^{\beta (A_{ik} - M)} \to 1 \,, \quad \sum_{j} e^{\beta (A_{ij} - M)} v_j \to v_{j^*} \,.
\]
It follows that $\lim_{\beta \to \infty} y_i(\beta) = v_{j^*}$. In other words, the output for query $i$ converges to the single value $v_{j^*}$ associated with the largest attention score in that row (the ``winning'' key). By definition of the tropical matrix product, this limiting value can be written as
\[
v_{j^*} = A_{i j^*} + v_{j^*} - A_{i j^*} = \max_{1 \le j \le n} \{ A_{ij} + v_j \} - M \,,
\]
since $A_{i j^*} = M$ and subtracting the constant $M$ does not affect which $j$ attains the maximum. (Geometrically, adding or subtracting $M$ corresponds to shifting all weights by a constant, which vanishes under $\max$.) Therefore $\lim_{\beta\to\infty} y_i(\beta) = \max_{j} \{ A_{ij} + v_j \}$. This holds for each $i$, so in vector form $\lim_{\beta\to\infty} (\softmax(\beta A)V) = A \otimes V$, as claimed.
\end{proof}

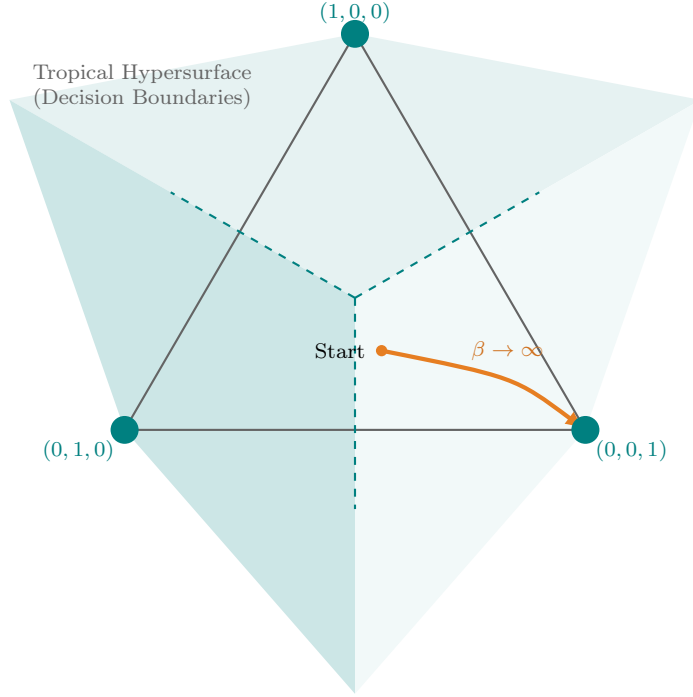
\begin{figure}[t]
\centering
\begin{tikzpicture}[scale=3.5, >=latex]
\coordinate (A) at (90:1);   
\coordinate (B) at (210:1);  
\coordinate (C) at (330:1);  
\coordinate (O) at (0,0);    

\fill[acadTeal!10] (O) -- (30:1.5) -- (A) -- (150:1.5) -- cycle;
\fill[acadTeal!20] (O) -- (150:1.5) -- (B) -- (270:1.5) -- cycle;
\fill[acadTeal!5]  (O) -- (270:1.5) -- (C) -- (30:1.5) -- cycle;

\draw[thick, acadGray] (A) -- (B) -- (C) -- cycle;

\draw[dashed, thick, acadTeal] (O) -- (30:0.8);
\draw[dashed, thick, acadTeal] (O) -- (150:0.8);
\draw[dashed, thick, acadTeal] (O) -- (270:0.8);

\coordinate (StartP) at (0.2, -0.4); 
\draw[->, ultra thick, acadOrange] (0.1, -0.2) .. controls (0.6, -0.3) .. (C);

\node[fill=acadOrange, circle, inner sep=1.5pt, label={left:\scriptsize Start}] at (0.1, -0.2) {};
\node[right, text=acadOrange!90!black, font=\bfseries\scriptsize] at (0.4, -0.2) {$\beta \to \infty$};

\fill[acadTeal] (A) circle (1.5pt) node[above] {\scriptsize $(1,0,0)$};
\fill[acadTeal] (B) circle (1.5pt) node[below left] {\scriptsize $(0,1,0)$};
\fill[acadTeal] (C) circle (1.5pt) node[below right] {\scriptsize $(0,0,1)$};

\node[align=center, font=\scriptsize, text=acadGray] at (-0.8, 0.8) {Tropical Hypersurface\\(Decision Boundaries)};

\end{tikzpicture}
\caption{Visualization of the tropical limit on the probability simplex. As the inverse temperature $\beta$ increases, a softmax probability distribution (orange path) concentrates toward a vertex (a one-hot vector). The dashed lines represent the \emph{tropical hypersurface} (decision boundaries) where coordinate scores are tied. These boundaries partition the simplex into regions where one component dominates. In the limit $\beta \to \infty$, the system flows deterministically to the vertex with the maximum score.}
\label{fig:simplex}
\end{figure}

\begin{corollary}[Multi-Layer Tropical Path]\label{cor:Llayers}
Consider $L$ successive self-attention layers in the tropical regime ($\beta \to \infty$). Let $V^{(0)}$ be the initial value vector (input token embeddings) and $Y^{(L)}$ the output after $L$ layers. Then
\[
Y^{(L)} = A^{\otimes L} \otimes V^{(0)} \,,
\]
where $A^{\otimes L}$ denotes the $L$-fold tropical power of $A$ (i.e., $A$ tropically multiplied with itself $L$ times). In particular, the $i$-th component of $Y^{(L)}$ is
\[
Y^{(L)}_i = \max_{j_0, \dots, j_{L-1}} \{ A_{i j_{L-1}} + A_{j_{L-1} j_{L-2}} + \dots + A_{j_1 j_0} + V^{(0)}_{j_0} \} \,,
\]
where the maximum is over all paths $j_0 \to j_1 \to \cdots \to j_{L-1} \to i$ of length $L$.
\end{corollary}

\begin{proof}
This is proved by induction on $L$, using Theorem~\ref{thm:self-attn-trop} for the base case $L=1$ and the associativity of tropical matrix multiplication to extend to $L>1$. The formula for $Y^{(L)}_i$ simply expands the tropical product explicitly. Notably, this expression is exactly the dynamic programming update for the best path of length $L$ from some starting index $j_0$ to the final index $i$. In other words, each attention layer performs one Bellman--Ford relaxation step on the weighted directed complete graph with adjacency matrix $A$, gradually building up the optimal path one hop at a time.
\end{proof}

\begin{figure}[t]
\centering
\begin{tikzpicture}[>=latex, node distance=2.5cm, auto]
\tikzset{
t_node/.style={circle, draw=acadTeal, thick, fill=white, minimum size=8mm, font=\bfseries\sffamily},
edge_norm/.style={->, thick, acadGray!70},
edge_opt/.style={->, ultra thick, acadOrange}
}

\node[t_node] (n0) at (0,0) {0};
\node[t_node] (n1) at (2.5, 1.5) {1};
\node[t_node] (n2) at (2.5, -1.5) {2};
\node[t_node] (n3) at (5,0) {3};

\draw[edge_norm, dashed] (n0) to[bend left=10] node[above, font=\small] {5} (n3);

\draw[edge_norm, dashed] (n0) to[bend right=20] node[below left, font=\small] {6} (n2);
\draw[edge_norm, dashed] (n2) to[bend right=20] node[below right, font=\small] {1} (n3);

\draw[edge_opt] (n0) to[bend left=20] node[above left, text=acadOrange] {\bfseries 4} (n1);
\draw[edge_opt] (n1) to[bend left=20] node[above right, text=acadOrange] {\bfseries 4} (n3);

\node[draw=acadTeal, rounded corners, fill=acadLight, text width=5.5cm, inner sep=6pt] 
at (10.5, 0) {
    \textbf{\footnotesize Pathfinding Logic (Target: Node 3)}\\
    \scriptsize
    \begin{itemize}
        \item[--] \textcolor{acadGray}{Direct ($0 \to 3$): $w=5$}
        \item[--] \textcolor{acadGray}{Via 2 ($0 \to 2 \to 3$): $w=6+1=7$}
        \item[--] \textcolor{acadOrange}{\textbf{Via 1 ($0 \to 1 \to 3$): $w=4+4=8$}}
    \end{itemize}
    \vspace{-2pt}
    \textit{\tiny The tropical product selects the max-weight path.}
};

\end{tikzpicture}
\caption{A toy illustration of tropical multi-hop reasoning. Each node represents a token. Directed edges are labeled with attention scores $A_{ij}$. In the tropical limit, a two-layer attention stack selects the path of length 2 maximizing the total weight. To compute the state at token 3, the model compares the direct edge (weight 5) against two-hop paths. The optimal path $0 \to 1 \to 3$ (highlighted in orange) has total weight $8$, surpassing the alternatives. This mirrors the Bellman--Ford algorithm.}
\label{fig:graphpath}
\end{figure}

\section{Discussion}
Our findings reveal a new perspective on how Transformers may carry out multi-step reasoning. In the high-confidence limit, self-attention no longer blends information from multiple tokens---it picks a single strongest connection. Consequently, an $L$-layer Transformer essentially selects an $L$-hop chain of tokens that maximizes a cumulative similarity score. This aligns with the intuition behind \emph{chain-of-thought}: the model is internally traversing a path of intermediate ``thought'' tokens that lead to a final answer. Each attention head can be seen as a pointer to the next step in a reasoning sequence, and deeper networks (or multiple reasoning steps allowed during prompting) enable longer chains. Recent theoretical results \cite{Hahn2020,Li2024} support this view by showing that without intermediate steps, a Transformer is limited in the complexity of functions it can compute, but allowing a sequence of $T$ steps (either via depth or prompting) can extend its computational power.

From a geometric standpoint, the tropical interpretation means the Transformer's computation is piecewise-linear and structured by \emph{tropical hypersurfaces}. Figure~\ref{fig:simplex} illustrated how increasing $\beta$ causes a probability distribution to concentrate on the largest component; the dashed lines in the simplex are where ties occur (the equations of the tropical hypersurface) \cite{Develin2004,Maclagan2015}. When attention selects the $\argmax$ at each layer, the network's state moves into one of the regions determined by these boundaries. In a high-dimensional token space, one can imagine a complex polyhedral partition where each region corresponds to a specific discrete chain-of-thought (a particular sequence of tokens dominating the attention at successive layers). The Transformer's forward pass in the tropical regime thus corresponds to traversing a vertex-to-vertex path on a tropical polytope defined by the attention scores. This path-finding view could potentially explain behaviors observed in large language models: for instance, why prompting a model to reason step-by-step \cite{Wei2022} or to ``think aloud'' might help, as it nudges the model's internal traversal to explicitly articulate intermediate nodes rather than making a single jump.

It is important to note that real Transformers operate at finite $\beta$ (finite confidence) and often with multiple attention heads that may explore different paths in parallel. Still, the tropical limit provides a skeletal picture of the underlying computation. Understanding how close practical models are to this regime, and whether they leverage approximate path-finding internally, is an exciting direction for future work. Connections between tropical optimization and neural network reasoning have been explored in other contexts (e.g., tropical geometry of deep ReLU networks \cite{Zhang2018}), suggesting a rich interface between algebraic geometry and machine learning theory. Our work specifically bridges idempotent mathematics \cite{Maslov1992,Litvinov2007} with modern AI reasoning mechanisms, contributing a piece to the puzzle of how large models can perform complex, multi-step inferences.

\bibliographystyle{plain}

\begin{thebibliography}{99}

\bibitem{Vaswani2017}
Ashish Vaswani, Noam Shazeer, Niki Parmar, Jakob Uszkoreit, Llion Jones, Aidan~N. Gomez, {\L}ukasz Kaiser, and Illia Polosukhin.
\newblock Attention is all you need.
\newblock In \emph{Advances in Neural Information Processing Systems (NIPS)}, volume~30, 2017.

\bibitem{Maslov1992}
Vladimir~P. Maslov and Stanislav~N. Samborski (Eds.).
\newblock \emph{Idempotent Analysis}.
\newblock Volume~13 of Advances in Soviet Mathematics. American Mathematical Society, 1992.

\bibitem{Kolokoltsov1997}
Vasily~N. Kolokoltsov and Vladimir~P. Maslov (Eds.).
\newblock \emph{Idempotent Analysis and Applications}.
\newblock Kluwer Academic Publishers, 1997.

\bibitem{Litvinov2007}
Grigori~L. Litvinov.
\newblock The {M}aslov dequantization, idempotent and tropical mathematics: a brief introduction.
\newblock \emph{Journal of Mathematical Sciences}, 140(3):426--444, 2007.

\bibitem{Cuninghame1979}
Raymond~A. Cuninghame-Green.
\newblock \emph{Minimax Algebra}, volume 166 of Lecture Notes in Economics and Mathematical Systems.
\newblock Springer, 1979.

\bibitem{Develin2004}
Mike Develin and Bernd Sturmfels.
\newblock Tropical convexity.
\newblock \emph{Documenta Mathematica}, 9:1--27, 2004.

\bibitem{Maclagan2015}
Diane Maclagan and Bernd Sturmfels.
\newblock \emph{Introduction to Tropical Geometry}.
\newblock Graduate Studies in Mathematics, Vol. 161. American Mathematical Society, 2015.

\bibitem{Baccelli1992}
Fran{\c{c}}ois Baccelli, Guy Cohen, Geert~Jan Olsder, and Jean-Pierre Quadrat.
\newblock \emph{Synchronization and Linearity: An Algebra for Discrete Event Systems}.
\newblock Wiley, 1992.

\bibitem{Butkovic2010}
Peter Butkovi\v{c}.
\newblock \emph{Max-linear Systems: Theory and Algorithms}.
\newblock Springer, 2010.

\bibitem{Wei2022}
Jason Wei, Xuezhi Wang, Dale Schuurmans, Maarten Bosma, Ed~Chi, Quoc~V. Le, and Denny Zhou.
\newblock Chain-of-thought prompting elicits reasoning in large language models.
\newblock \emph{arXiv:2201.11903}, 2022.

\bibitem{Hahn2020}
Michael Hahn.
\newblock Theoretical limitations of self-attention in neural sequence models.
\newblock \emph{Transactions of the Association for Computational Linguistics}, 8:156--171, 2020.

\bibitem{Merrill2023}
William Merrill and Ashish Sabharwal.
\newblock The parallelism tradeoff: Limitations of log-precision transformers.
\newblock \emph{Transactions of the Association for Computational Linguistics}, 11:531--545, 2023.

\bibitem{Li2024}
Zhiyuan Li, Hong Liu, Denny Zhou, and Tengyu Ma.
\newblock Chain of thought empowers transformers to solve inherently serial problems.
\newblock In \emph{Proc. of International Conference on Learning Representations (ICLR)}, 2024.

\bibitem{Bellman1958}
Richard Bellman.
\newblock On a routing problem.
\newblock \emph{Quarterly of Applied Mathematics}, 16(1):87--90, 1958.

\bibitem{Ford1956}
Lester~R. Ford, Jr.
\newblock Network flow theory.
\newblock Technical Report P-923, The RAND Corporation, 1956.

\bibitem{Mohri2002}
Mehryar Mohri.
\newblock Semiring frameworks and algorithms for shortest-distance problems.
\newblock \emph{Journal of Automata, Languages and Combinatorics}, 7(4):321--350, 2002.

\bibitem{Jang2017}
Eric Jang, Shixiang Gu, and Ben Poole.
\newblock Categorical reparameterization with {G}umbel--{Softmax}.
\newblock In \emph{Proc. of International Conference on Learning Representations (ICLR)}, 2017.

\bibitem{Zhang2018}
Liwen Zhang, Gregory Naitzat, and Lek-Heng Lim.
\newblock Tropical geometry of deep neural networks.
\newblock In \emph{Proc. of the 35th International Conference on Machine Learning (ICML)}, pages 5824--5832, 2018.

\bibitem{Vorobyev1967}
Nikolai~N. Vorob'ev.
\newblock Extremal algebra of positive matrices.
\newblock \emph{Elektronische Informationsverarbeitung und Kybernetik}, 3:39--71, 1967.

\bibitem{Goodman1999}
Joshua Goodman.
\newblock Semiring parsing.
\newblock \emph{Computational Linguistics}, 25(4):573--605, 1999.

\end{thebibliography}

\end{document}